\documentclass{article}
\usepackage{amssymb}
\usepackage{amsmath}
\usepackage{amsthm}
\usepackage{mathrsfs}
\usepackage{footnote}
\usepackage{color}
\usepackage{authblk}
\usepackage[usenames,dvipsnames,svgnames,table]{xcolor}

\newif\iftoonzichtbaar
\newif\iftoonverborgen

\toonzichtbaarfalse
\toonverborgentrue

\newtheoremstyle{theorem}{7pt}{3pt}{\itshape}{}{\bfseries}{}{\newline}{}
\theoremstyle{theorem}
\newtheorem{theo}{Theorem}
\newtheorem{proposition}[theo]{Proposition}
\newtheorem{lemma}[theo]{Lemma}
\newtheoremstyle{definition}{7pt}{5pt}{\itshape}{}{\bfseries}{}{\newline}{}
\theoremstyle{definition}
\newtheorem{definition}{Definition}
\newtheoremstyle{example}{6pt}{3pt}{}{}{\bfseries}{}{\newline}{}
\theoremstyle{example}
\newtheorem{example}{Example}
\newtheoremstyle{note}{5pt}{3pt}{}{}{\bfseries}{}{\newline}{}
\theoremstyle{note}

\newcommand{\la}{\leftarrow}

\newcommand{\Ra}{\Rightarrow}
\newcommand{\PP}{\mathcal{P}}
\newcommand{\es}{\emptyset}
\newcommand{\dpt}{\hspace{-2pt}:\hspace{-2pt}}

\title{Modeling Stable Matching Problems with Answer Set Programming\thanks{This research was funded by a Research Foundation-Flanders project.}} 
\author[a]{Sofie De Clercq}
\author[b]{Steven Schockaert}
\author[a]{Martine De Cock}
\author[c]{Ann Now\'e}
\affil[a]{{\scriptsize Department of Applied Mathematics, Computer Science \& Statistics, Ghent University, Krijgslaan 281, 9000 Ghent, \underline{Belgium} \newline \textbf{\{Sofier.DeClercq,Martine.DeCock\}@ugent.be}}}
\affil[b]{{\scriptsize School of Computer Science \& Informatics, Cardiff University, 5 The Parade, Roath, Cardiff CF24 3AA, \underline{United Kingdom} \newline \textbf{S.Schockaert@cs.cardiff.ac.uk}}}
\affil[c]{{\scriptsize Computational Modeling Lab, Vrije Universiteit Brussel, Pleinlaan 2, 1050 Brussel, \underline{Belgium} \newline \textbf{ANowe@vub.ac.be}}}
\date{}

\begin{document}

\pagestyle{empty}

\mbox{ }
\vspace{4cm}
\begin{center}
{\Huge Modeling Stable Matching Problems}
\end{center}
\begin{center}
{\Huge with Answer Set Programming}
\end{center}
\vspace{1cm}
\begin{center}
{\Large Sofie De Clercq\textsuperscript{a}, Steven Schockaert\textsuperscript{b}, Martine De Cock\textsuperscript{a} \& Ann Now\'e\textsuperscript{c}}
\end{center}
\begin{center}
\textsuperscript{a} {\small Department of Applied Mathematics, Computer Science \& Statistics, Ghent University, Krijgslaan 281, 9000 Ghent, \underline{Belgium}  \textbf{\{Sofier.DeClercq,Martine.DeCock\}@ugent.be}}\\
\textsuperscript{b} {\small School of Computer Science \& Informatics, Cardiff University, 5 The Parade, Roath, Cardiff CF24 3AA, \underline{United Kingdom} \textbf{S.Schockaert@cs.cardiff.ac.uk}}\\
\textsuperscript{c} {\small Computational Modeling Lab, Vrije Universiteit Brussel, Pleinlaan 2, 1050 Brussel, \underline{Belgium}  \textbf{ANowe@vub.ac.be}}
\end{center}
\vspace{3cm}
\begin{center}
{\LARGE \textbf{The final publication is available at}}\\
\vspace{0.5cm}
{\LARGE \textbf {link.springer.com}}
\end{center}
\vspace{0.2cm}

\newpage
\maketitle

\begin{abstract}
The Stable Marriage Problem (SMP) is a well-known matching problem first introduced and solved by Gale and Shapley~\cite{SMPGale62}. Several variants and extensions to this problem have since been investigated to cover a wider set of applications. Each time a new variant is considered, however, a new algorithm needs to be developed and implemented. As an alternative, in this paper we propose an encoding of the SMP using Answer Set Programming (ASP). Our encoding can easily be extended and adapted to the needs of specific applications. As an illustration we show how stable matchings can be found when individuals may designate unacceptable partners and ties between preferences are allowed. Subsequently, we show how our ASP based encoding naturally allows us to select specific stable matchings which are optimal according to a given criterion. Each time, we can rely on generic and efficient off-the-shelf answer set solvers to find (optimal) stable matchings.
\end{abstract}

\abovedisplayskip=4pt
\belowdisplayskip=4pt
\section{Introduction}
The Stable Marriage Problem (SMP) is a matching problem first introduced and solved by Gale and Shapley~\cite{SMPGale62}. Starting from (i) a set of $n$ men and $n$ women, (ii) for each man a ranking of the women as preferred partners, and (iii) for each woman a ranking of the men as preferred partners, the SMP searches for a set of $n$ couples (marriages) such that there are no man and woman who are in different marriages but both prefer each other to their actual partners. Such a man and woman are called a \textit{blocking pair} and a matching without blocking pairs forms a \textit{stable set} of marriages.
Due to its practical relevance, countless variants on the SMP have been investigated, making the problem assumptions more applicable to a wider range of applications, such as kidney-exchange~\cite{SMPIrv07} and the hospital-resident problem~\cite{SMPMan02}. Recently Roth and Shapley won the Nobel Prize for Economics for their theory of stable allocations and the practice of market design, work that has directly resulted from an application of the SMP.

In the literature, typically each time a new variant on the SMP is considered, a new algorithm is developed (see e.g. \cite{SMPGus87,SMPIrv87,SMPMc12}). In this paper, we propose to use Answer Set Programming (ASP) as a general vehicle for modeling a large class of extensions and variations of the SMP. We show how an ASP encoding allows us to express in a natural way ties in the preferences of men and women, as well as unacceptability constraints (where certain people prefer to remain single over being coupled to undesirable partners). Furthermore, we illustrate how we can use our ASP encoding to find stable matchings that are optimal according to a certain criterion. Although the SMP has been widely investigated, and efficient approximation or exact algorithms are available for several of its variants (see e.g. \cite{SMPMc12}), to the best of our knowledge, our encoding offers the first exact implementation to find sex-equal, minimum regret, egalitarian or maximum cardinality stable sets for SMP instances with unacceptability and ties. 

The paper is structured as follows. In Section \ref{sect:background} we give some background about the SMP and ASP. We introduce our encoding of the SMP with ASP and prove its correctness in the third section. In Section \ref{sec:oss}, we extend our encoding enabling it to find optimal stable sets. We explore several notions of optimality for stable matchings and show how optimal stable matchings can be found by solving the corresponding disjunctive ASP program. Finally we draw our conclusions. 

\section{Background}\label{sect:background}
\subsection{The Stable Marriage Problem}
To solve the standard SMP, Gale and Shapley~\cite{SMPGale62} constructed an iterative algorithm ---known as the Gale-Shapley algorithm, G-S algorithm or deferred-acceptance algorithm--- to compute a particular solution of an SMP instance. The algorithm works as follows: in round 1 every man proposes to his first choice of all women. A woman, when being proposed, then rejects all men but her first choice among the subset of men who proposed to her. That first choice becomes her temporary husband. In the next rounds, all rejected men propose to their first choice of the subset of women by whom they were not rejected yet, regardless of whether this woman already has a temporary husband. Each woman, when being proposed, then rejects all men but her first choice among the subset of men who just proposed to her and her temporary mate. This process continues until all women have a husband. This point, when everyone has a partner, is always reached after a polynomial number of steps and the corresponding set of marriages is stable~\cite{SMPGale62}. It should be noted, however, that only one of the potentially exponentially many stable matchings is found in this way. We formally define the SMP and introduce two variants that will be considered in this paper. We denote a set of men as $M=\{m_1,\hdots,m_n\}$ and a set of women $W=\{w_1,\hdots,w_p\}$, with $n=p$ for the classical SMP. A set of marriages is a set of man-woman pairs such that each man and each woman occurs in just one pair.
\begin{definition}[Classical SMP] \label{def:smp}
An instance of the classical SMP is a pair $(S_M,S_W)$, with $S_M = \{\sigma_M^1,\hdots,\sigma_M^n\}$ and $S_W = \{\sigma_W^1,\hdots,\sigma_W^n\}$ sets of permutations of the integers $1,\hdots,n$. The permutations $\sigma_M^i$ and $\sigma_W^i$ are the preferences of man $m_i$ and woman $w_i$ respectively. If $\sigma_M^i(j)=k$, we say that woman $w_k$ is the $j^{th}$ most preferred woman for man $m_i$, and similarly for $\sigma_W^i(j)=k$. Man $m$ and woman $w$ form a blocking pair in a set of marriages $S$ if $m$ prefers $w$ to his partner in $S$ and $w$ prefers $m$ to her partner in $S$. A solution of an instance is a stable set of marriages, i.e.\ a set of marriages without blocking pairs.
\end{definition}

A first variant of the classical SMP allows men and women to point out unacceptable partners by not including them in their preference list. The number of men $n$ can differ from the number of women $p$ since men and women can remain single. A set of marriages is a set of singles (i.e.\ persons paired to themselves) and man-woman pairs such that every man and woman occurs in just one pair. 
\begin{definition}[SMP with unacceptability] \label{def:smpunacc}
An instance of the SMP with unacceptability is a pair $(S_M,S_W)$, $S_M = \{\sigma_M^1,\hdots,\sigma_M^n\}$, and $S_W = \{\sigma_W^1$, $\hdots$, $\sigma_W^p\}$, with each $\sigma_M^i$ a permutation of a subset of $\{1,\hdots,p\}$ and each $\sigma_W^j$ a permutation of a subset of $\{1,\hdots,n\}$. If $\sigma_M^i(j)=k$, woman $w_k$ is the $j^{th}$ most preferred woman for man $m_i$, and similarly for $\sigma_W^i(j)=k$. If there is no $l$ such that $\sigma_M^i(l)=j$, woman $w_j$ is an unacceptable partner for man $m_i$, and similarly for no $l$ such that $\sigma_W^i(l)=j$. A person $x$ forms a blocking individual in a set of marriages $S$ if $x$ prefers being single to being paired with his or her partner in $S$. A solution of an instance is a stable set of marriages, i.e.\ a set of marriages without blocking pairs or individuals.
\end{definition}
The length of the permutation $\sigma_M^i$ is denoted as $|\sigma_M^i|$. A stable matching for an SMP instance with unacceptability always exists and can be found in polynomial time~\cite{SMPRoth90} by a slightly modified G-S algorithm.
\begin{example} \label{ex:unacc}
Suppose $M=\{m_1,m_2,m_3\}$, $W=\{w_1,w_2,w_3,w_4\}$, $S_M = \{ \sigma_M^1= (4,1,3), \sigma_M^2 = (3,2), \sigma_M^3 = (1,3)\}$ and $S_W = \{ \sigma_W^1 = (1,3), \sigma_W^2 = (2), \sigma_W^3 = (3,2), \sigma_W^4 = (2,1)\}$. Hence woman $w_1$ prefers man $m_1$ to man $m_3$ while man $m_2$ is unacceptable. In this setting, there is exactly one stable set of marriages~\cite{SMPRoth90}: $\{(m_1,w_4),(m_2,w_3),(m_3,w_1),(w_2,w_2)\}$. Thus woman $w_2$ stays single.
\end{example}

The second variant of the SMP allows unacceptability and ties, i.e.\ the preferences do not have to be strict. For this variant there are several ways to define stability, but we will use the notion of weak stability~\cite{SMPIrv94}.
\begin{definition}[SMP with unacceptability and ties] \label{def:smpindif}
An instance of the SMP with unacceptability and ties is a pair $(S_M,S_W)$, $S_M = \{\sigma_M^1,\hdots,\sigma_M^n\}$ and $S_W = \{\sigma_W^1,\hdots,\sigma_W^p\}$. For every $i\in\{1,\hdots,n\}$, $\sigma_M^i$ is a list of disjoint subsets of $\{1,\hdots,p\}$. Symmetrically $\sigma_W^i$ is a list of disjoint subsets of $\{1,\hdots,n\}$ for every $i\in\{1,\hdots,p\}$. We call $\sigma_M^i$ and $\sigma_W^i$ the preferences of man $m_i$ and woman $w_i$ respectively. If $k \in \sigma_M^i(j)$, woman $w_k$ is in man $m_i$'s $j^{th}$ most preferred group of women. All the women in that group are equally preferred by $m_i$. The case $k \in \sigma_W^i(j)$ is similar. If there is no $l$ such that $j \in \sigma_M^i(l)$, woman $w_j$ is an unacceptable partner for man $m_i$, and similar for no $l$ such that $j \in \sigma_W^i(l)$. For every $k$ in the set\footnote[1]{$|\sigma_M^i|$ denotes the length of the list $\sigma_M^i$.} $\sigma_M^i(|\sigma_M^i|)$, man $m_i$ equally prefers staying single to being paired to woman $w_k$, and symmetrically for the preferences of a woman $w_i$. This is the only set in $\sigma_M^i$ that might be empty, and similar for $\sigma_W^i$. 
Man $m$ and woman $w$ form a blocking pair in a set of marriages $S$ if $m$ strictly prefers $w$ to his partner in $S$ and $w$ strictly prefers $m$ to her partner in $S$. A blocking individual in $S$ is a person who stricly prefers being single to being paired to his partner in $S$.
A solution of an instance is a weakly stable set of marriages, i.e.\ a set of marriages without blocking pairs or individuals. 
\end{definition}
A weakly stable matching always exists for an instance of the SMP with unacceptability and ties and it can be found in polynomial time by arbitrarily breaking the ties~\cite{SMPIwa08}. However, as opposed to the previous variant, the number of matched persons is no longer constant for every stable set in this variant. Note that the setting of Definition \ref{def:smpindif} generalizes the setting of Definition \ref{def:smpunacc}, which generalizes the setting of Definition \ref{def:smp}. 
We introduce the notations
\[
acceptable_M^i = \underbrace{\sigma^i_M(1) \cup \sigma^i_M(2) \cup \ldots \cup \sigma^i_M(|\sigma^i_M|-1)}_{\displaystyle = \mbox{\textit{preferred}}_M^i} \cup \underbrace{\sigma^i_M(|\sigma^i_M|)}_{\displaystyle = neutral_M^i}
\]
Furthermore $unacceptable_M^i = \{1,\hdots,p\} \setminus acceptable_M^i$. We define the ordening $\leq_M^{m_i}$ on $\{w_j\,|\, j \in acceptable_M^i\} \cup \{m_i\}$ as $x \leq_M^{m_i} y$ iff $m_i$ prefers person $x$ at least as much as person $y$. The strict ordening $<_M^{m_i}$ is defined in the obvious way and analogous notations are used for $\sigma_W^j$.
\begin{example} \label{ex:indif}
Suppose $M=\{m_1$, $m_2\}$, $W=\{w_1$, $w_2$, $w_3$, $w_4\}$ and $S_M = \{ \sigma_M^1$ = $(\{1,3\},\{4\})$, $\sigma_M^2$ = $(\{2,3\},\{\})\}$. Hence man $m_1$ prefers women $w_1$ and $w_3$ to woman $w_4$. There is a tie between woman $w_1$ and $w_3$ as well as between woman $w_4$ and staying single. Woman $w_2$ is unacceptable for man $m_1$. Man $m_2$ prefers woman $w_2$ and $w_3$ to staying single, but finds $w_1$ and $w_4$ unacceptable. It holds that $w_1 <_M^{m_1} m_1$, i.e.~$m_1$ prefers marrying $w_1$ over staying single, \textit{acceptable}$_M^1=\{1,3,4\}$, \textit{preferred}$_M^1=\{1,3\}$, $neutral_M^1= \{4\}$ and $unacceptable_M^1=\{2\}$.
\end{example}

\subsection{Answer Set Programming}
Answer Set Programming or ASP is a form of declarative programming~\cite{ASPBrew11}. Its transparence, elegance and ability to deal with $\Sigma_2^P$-complete problems make it an attractive method for solving combinatorial search and optimization problems. An ASP program is a finite collection of first-order rules
\begin{align*}
A_1 \vee \hdots \vee A_k \la B_1,\hdots,B_m,not \,C_1,\hdots, not \, C_n
\end{align*}
with $A_1, \hdots, A_k, B_1,\hdots,B_m,C_1,\hdots,C_n$ predicates.
The semantics are defined by the \textit{ground version} of the program, consisting of all ground instantiations of the rules w.r.t.\ the constants that appear in it (see e.g. \cite{ASPBrew11} for a good overview). This grounded program is a propositional ASP program. The building blocks of these programs are \textit{atoms}, \textit{literals} and \textit{rules}. The most elementary are \textit{atoms}, which are propositional variables that can be true or false. A \textit{literal} is an atom or a negated atom. Beside strong negation, ASP uses a special kind of negation, namely \textit{negation-as-failure} (naf), denoted with `$not$'. For a literal $a$ we call `$not \, a$' the naf-literal associated with $a$. The \textit{extended literals} consist of all literals and their associated naf-literals. A \textit{disjunctive rule} has the following form
\begin{align*}
a_1 \vee \hdots \vee a_k \la b_1,\hdots,b_m,not \,c_1,\hdots, not \, c_n
\end{align*}
where $a_1, \hdots, a_k, b_1,\hdots,b_m,c_1,\hdots,c_n$ are literals from a fixed set $\mathcal{L}$, determined by a fixed set $\mathcal{A}$ of atoms. We call $a_1 \vee \hdots \vee a_k$ the head of the rule while the set of extended literals $b_1,\hdots,b_m,not \,c_1,\hdots$, $not \, c_n$ is called the \textit{body}.
The rule above intuitively encodes that $a_1$, $a_2$, $\hdots$ or $a_k$ is true when we have evidence that $b_1,\hdots,b_m$ are true and we have no evidence that at least one of $c_1,\hdots,c_n$ are true. When a rule has an empty body, we call it a \textit{fact}; when the head is empty, we speak of a \textit{constraint}. A rule without occurrences of $not$ is called a \textit{simple disjunctive rule}. A \textit{simple disjunctive} ASP program is a finite collection of simple disjunctive rules and similarly a \textit{disjunctive} ASP program $\PP$ is a finite collection of disjunctive rules. If each rule head consists of at most one literal, we speak of a \textit{normal} ASP program. 

We define an \textit{interpretation} $I$ of a disjunctive ASP program $\PP$ as a subset of $\mathcal{L}$. An interpretation $I$ \textit{satisfies} a simple disjunctive rule $a_1 \vee \hdots \vee a_k$ $\la b_1,\hdots,b_m$ when $a_1 \in I \vee \hdots \vee a_k \in I$ or $\{b_1,\hdots,b_m\} \not \subseteq I$. An interpretation which satisfies all rules of a simple disjunctive program is called a \textit{model} of that program. 
An interpretation $I$ is an \textit{answer set} of a simple disjunctive program $\PP$ iff it is a minimal model of $\PP$, i.e.\ no strict subset of $I$ is a model of $\PP$~\cite{ASPGel88}.
The \textit{reduct} $\PP^I$ of a disjunctive ASP program $\PP$ w.r.t.\ an interpretation $I$ is defined as the simple disjunctive ASP program $\PP^I=\{a_1 \vee \hdots \vee a_k \la b_1,\hdots,b_m \,|\, (a_1 \vee \hdots \vee a_k \la b_1,\hdots,b_m,not \,c_1,\hdots, not \, c_n) \in \PP, \{c_1,\hdots,c_n\} \cap I = \es \}$. An interpretation $I$ of a disjunctive ASP program $\PP$ is an answer set of $\PP$ iff $I$ is an answer set of $\PP^I$.
\begin{example}
Let $\PP$ be the ASP program with the following 4 rules:
\begin{align*}
man(john) &\la,\quad person(john)\la,\quad person(fiona)\la \\
woman(X) \vee child(X) &\la person(X), not\, man(X)
\end{align*}
The last rule is grounded to 2 rules in which $X$ is resp.\ replaced by $john$ and by $fiona$. We check that the interpretation $I=\{man(john), woman(fiona)$, $person(john)$, $person(fiona)\}$ is an answer set of the ground version of $\PP$ by computing the reduct. The grounded rule with $X=john$ is deleted since $man(john)$ is in $I$. The reduct $\PP^I$ is:
\allowdisplaybreaks
\begin{align*}
man(john) &\la,\quad person(john)\la, \quad person(fiona)\la \\
woman(fiona) \vee child(fiona) &\la person(fiona)
\end{align*}
The first 3 rules are facts, hence their heads will be in any answer set. The fourth rule encodes that any person who is not a man, is a woman or child. It is clear that $I$ is a minimal model of this simple program, so $I$ is an answer set of $\PP$. By replacing $woman(fiona)$ by $child(fiona)$ in $I$, another answer set is obtained.
\end{example}

To automatically compute the answer sets of the programs in this paper, we have used the ASP solver DLV\footnote[2]{Available from www.dlvsystems.com}, due to its ability to handle predicates, disjunction and numeric values (with some built-in aggregate functions). The numeric values are only used for grounding.

\section{Modeling the Stable Marriage Problem in ASP} \label{sect:SMPinASP}
In this section we model variations and generalizations of the SMP with ASP. A few proposals of using nonmonotonic reasoning for modeling the SMP have already been described in the literature. For instance in~\cite{ASPMa90} a specific variant of the SMP is mentioned (in which boys each know a subset of a set of girls and want to be matched to a girl they know) and in~\cite{SMPDung95} an abductive program is used to find a stable set of marriages in which two fixed persons are paired, with strict, complete preference lists. To the best of our knowledge, beyond a few specific examples, no comprehensive study has been made of using ASP or related paradigms in this context. In particular, the generality of our ASP framework for weakly stable sets of SMP instances with unacceptablity and/or ties is a significant advantage.
The expression $accept(m,w)$ denotes that a man $m$ and a woman $w$ accept each other as partners. The predicate $manpropose(m,w)$ expresses that man $m$ is willing to propose to woman $w$ and analogously $womanpropose(m,w)$ expresses that woman $w$ is willing to propose to man $m$. Inspired by the Gale-Shapley algorithm, we look for an ASP formalisation to find the stable sets.
\begin{definition}[ASP program induced by SMP with unacc.\ and ties] \label{def:aspsmpindif}
The ASP program $\PP$ induced by an instance $(\{\sigma_M^1,\hdots,\sigma_M^n\},\{\sigma_W^1,\hdots,\sigma_W^p\})$ of the classical SMP with unacceptability and ties is the program containing for every $i\in\{1,\hdots,n\},j\in\{1,\hdots,p\}$ the following rules:
\begin{align}
accept(m_i,w_j) &\la manpropose(m_i,w_j), womanpropose(m_i,w_j) \label{eq:ruleacc}\\
accept(m_i,m_i) &\la \{not\, accept(m_i,w_k) \,|\, k \in acceptable_M^i\} \label{eq:msingleindif} \\
accept(w_j,w_j) &\la \{not\, accept(m_k,w_j) \,|\, k \in acceptable_W^j\} \label{eq:wsingleindif}
\end{align}
and for every $i\in\{1,\hdots,n\}$, $j\in acceptable_M^i$:
\begin{align}
manpropose(m_i,w_j) &\la \{not\, accept(m_i,x)\,|\, x \leq_M^{m_i} w_j  \mbox{ and } w_j \neq x\} \label{eq:mpropindif}
\end{align}
and for every $j\in\{1,\hdots,p\}$, $i\in acceptable_W^j$:
\begin{align}
womanpropose(m_i,w_j) &\la \{not\, accept(x,w_j)\,|\, x \leq_W^{w_j} m_i  \mbox{ and } m_i \neq x\} \label{eq:wpropindif}
\end{align}
\end{definition}
Intuitively (\ref{eq:ruleacc}) means that a man and woman accept each other as partners if they propose to each other. Due to (\ref{eq:msingleindif}), a man accepts himself as a partner (i.e.\ stays single) if no woman in his preference list is prepared to propose to him. Rule (\ref{eq:mpropindif}) states that a man proposes to a woman if he is not paired to a more or equally preferred woman. For $j \in neutral_M^i$ the body of (\ref{eq:mpropindif}) contains $not \, accept(m_i,m_i)$. 
No explicite rules are stated about the number of persons someone can propose to or accept but Proposition \ref{pr:SMPASPindifasss} implies that this is unnecessary.

We illustrate the induced ASP program with an example.
\begin{example} \label{ex:smpasp}
Consider the following instance $(S_M,S_W)$ of the SMP with unacceptability and ties.
Let $M=\{m_1,m_2\}$ and $W=\{w_1,w_2,w_3\}$. 
Furthermore:
\begin{align*}
\sigma_M^1 &= (\{1\},\{2,3\},\{\})\\
\sigma_M^2 &= (\{2\},\{1\})\\
\sigma_W^1 &= (\{1,2\},\{\})\\
\sigma_W^2 &= (\{1\},\{\})\\
\sigma_W^3 &= (\{2\},\{1\},\{\})
\end{align*}
The ASP program induced by this SMP instance is:
\allowdisplaybreaks
\begin{align*}
man(m_1) &\la, \quad man(m_2) \la \\
woman(w_1) &\la, \quad, woman(w_2) \la, \quad woman(w_3) \la\\
accept(X,Y) &\la manpropose(X,Y),womanpropose(X,Y), man(X), woman(Y)\\
manpropose(m_1,w_1) &\la \\
manpropose(m_1,w_2) &\la not\, accept(m_1,w_1),not\, accept(m_1,w_3)\\
manpropose(m_1,w_3) &\la not\, accept(m_1,w_1),not\, accept(m_1,w_2)\\
accept(m_1,m_1) &\la not\, accept(m_1,w_1),not\, accept(m_1,w_2),not\, accept(m_1,w_3)\\
manpropose(m_2,w_2) &\la \\
manpropose(m_2,w_1) &\la  not\,accept(m_2,w_2),not\,accept(m_2,m_2)\\ 
accept(m_2,m_2) &\la not\,accept(m_2,w_2),not\,accept(m_2,w_1)\\
womanpropose(m_1,w_1) &\la not\,accept(m_2,w_1)\\
womanpropose(m_2,w_1) &\la not\,accept(m_1,w_1)\\ 
accept(w_1,w_1)  &\la not\,accept(m_1,w_1), not\,accept(m_2,w_1)\\
womanpropose(m_1,w_2) &\la\\
accept(w_2,w_2)  &\la not\,accept(m_1,w_2)\\
womanpropose(m_2,w_3) &\la \\
womanpropose(m_1,w_3) &\la not\, accept(m_2,w_3)\\
accept(w_3,w_3)  &\la not\,accept(m_1,w_3), not\,accept(m_2,w_3)
\end{align*}
Notice that we use the facts $man$ and $woman$ to capture all the rules of the form (\ref{eq:ruleacc}) at once.
If we run this program in DLV, we get three answer sets containing respectively:
\begin{itemize}
\item $\{accept(m_1,w_3), accept(m_2,w_1), accept(w_2,w_2)\}$,
\item $\{accept(m_1,w_2), accept(m_2,w_1), accept(w_3,w_3)\}$,
\item $\{accept(m_1,w_1), accept(m_2,m_2), accept(w_2,w_2), accept(w_3,w_3)\}$.
\end{itemize}
These correspond to the three weakly stable set of marriages of this SMP instance, namely $\{(m_1,w_3)$, $(m_2,w_1)$, $(w_2,w_2)\}$, $\{(m_1,w_2)$, $(m_2,w_1)$, $(w_3,w_3)\}$ and $\{(m_1,w_1)$, $(m_2,m_2)$, $(w_2,w_2),(w_3,w_3)\}$.
\end{example}

\begin{proposition} \label{pr:SMPASPindifasss}
Let $(S_M,S_W)$ be an instance of the SMP with unacceptability and ties and let $\PP$ be the corresponding ASP program. If $I$ is an answer set of $\PP$, then a weakly stable matching for $(S_M,S_W)$ is given by $\{(x,y) \,|\, accept(x,y)\in I\}$.
\end{proposition}
\begin{proof}
Let $(S_M,S_W)$ and $\PP$ be as described in the proposition. Because of the symmetry between the men and the women we restrict ourselves to the male case when possible. We prove this proposition in 4 steps.
\begin{enumerate}
\item \textit{For every $i\in \{1,\hdots,n\}$, every $j \in \{1,\hdots,p\}$ and for every answer set $I$ of $\PP$, it holds that $accept(m_i,w_j) \in I$ implies that $j \in acceptable_M^i$ and $i \in acceptable_W^j$.}\\
This can be proved by contradiction. We will prove that for every man $m_i$ and every $j \in unacceptable^i_M$, $accept(m_i,w_j)$ is in no answer set $I$ of the induced ASP program $\PP$. For $accept(m_i,w_j)$ to be in an answer set $I$, the reduct must contain some rule with this literal in the head and a true body. The only rule which can make this happen is the one of the form (\ref{eq:ruleacc}), implying that $manpropose(m_i,w_j)$ should be in $I$. But since $j$ is not in $acceptable^i_M$ there is no rule with $manpropose(m_i,w_j)$ in the head and so $manpropose(m_i,w_j)$ can never be in $I$.
\item \textit{For every answer set $I$ of $\PP$ and every man $m_i$, there exists at most one woman $w_j$ such that $accept(m_i,w_j)\in I$. Similarly, for every woman $w_j$ there exists at most one man $m_i$ such that $accept(m_i,w_j)\in I$. Moreover, if $accept(m_i,m_i)\in I$ then $accept(m_i,w_j)\notin I$ for any $w_j$, and likewise when $accept(w_j,w_j)\in I$ then $accept(m_i,w_j)\notin I$ for any $m_i$.}\\
This can be proved by contradiction. Suppose first that there is an answer set $I$ of $\PP$ that contains $accept(m_i,w_j)$ and $accept(m_i,w_{j'})$ for some man $m_i$ and two different women $w_j$ and $w_{j'}$. The first step implies that $j$ and $j'$ are elements of $acceptable_M^i$. Either man $m_i$ prefers woman $w_j$ to woman $w_{j'}$ ($w_j \leq_M^{m_i} w_{j'}$) or the other way around ($w_{j'} \leq_M^{m_i} w_{j}$) or man $m_i$ has no preference among them ($w_j \leq_M^{m_i} w_{j'}$ and $w_{j'} \leq_M^{m_i} w_{j}$). The first two cases are symmetrical and can be handled analogously. The last case follows from the first case because it has stronger assumptions. We prove the first case and assume that man $m_i$ prefers woman $w_j$ to woman $w_{j'}$. The rules (\ref{eq:mpropindif}) imply the presence of a rule $manpropose(m_i,w_{j'})$ $\la \hdots$, $not \, accept(m_i,w_j), \hdots$ and this is the only rule which can make $manpropose(m_i,w_{j'})$ true (the only rule with this literal in the head). However, since $accept(m_i,w_j)$ is also in the answer set, this rule has a false body so $manpropose(m_i,w_{j'})$ can never be in $I$. Consequently $accept(m_i,w_{j'})$ can never be in $I$ since the only rule with this literal in the head is of the form (\ref{eq:ruleacc}) and this body can never be true, which leads to a contradiction.\\
Secondly assume that $accept(m_i,w_j)$ and $accept(m_i,m_i)$ are both in an answer set $I$ of $\PP$. Again step 1 implies that $j \in acceptable^i_M$. Because of the rules (\ref{eq:msingleindif}) $\PP$ will contain the rule $accept(m_i,m_i) \la$ $\hdots, not \, accept(m_i,w_j), \hdots$. An analogous reasoning as above implies that since $accept(m_i,w_j)$ is in the answer set $I$, $accept(m_i,m_i)$ can never be in $I$.
\item \textit{For every man $m_i$, in every answer set $I$ of $\PP$ exactly one of the following conditions is satisfied}:
\begin{enumerate}
\item \textit{there exists a woman $w_j$ such that $accept(m_i,w_j) \in I$,}
\item $accept(m_i,m_i) \in I$,
\end{enumerate} 
\textit{and similarly for every woman $w_i$.} \\
Suppose $I$ is an arbitrary answer set of $\PP$ and $m_i$ is an arbitrary man. We already know from step 2 that a man cannot be paired to a woman while being single, so both possibilities are disjoint. So suppose there is no woman $w_j$ such that $accept(m_i,w_j)$ is in $I$. $\PP$ will contain the rule (\ref{eq:msingleindif}). Because of our assumptions and the definition of the reduct, this rule will be reduced to $accept(m_i,m_i) \la$, and so $accept(m_i,m_i)$ will be in $I$.
\item For an arbitrary answer set $I$ of $\PP$ the previous steps imply that $I$ produces a set of marriages without blocking individuals. Weak stability also demands the absence of blocking pairs. Suppose by contradiction that there is a blocking pair $(m_i,w_j)$, implying that there exist $i\neq i'$ and $j \neq j'$ such that $accept(m_i,w_{j'}) \in I$ and $accept(m_{i'},w_j) \in I$ while $w_{j} <_M^{m_i} w_{j'}$ and $m_i <_W^{w_j} m_{i'}$. The rules of the form (\ref{eq:ruleacc}), the only ones with the literals $accept(m_i,w_{j'})$ and $accept(m_{i'},w_j)$ in the head, imply that literals $manpropose(m_i,w_{j'})$ and $womanpropose(m_{i'},w_j)$ should be in $I$. But since $w_{j} <_M^{m_i} w_{j'}$ and because of the form of the rules (\ref{eq:mpropindif}) there are fewer conditions to be fulfilled for $manpropose(m_i,w_j)$ to be in $I$ than for $manpropose(m_i,w_{j'})$ to be in $I$. So $manpropose(m_i,w_j)$ should be in $I$ as well. A similar reasoning implies that $womanpropose(m_i,w_j)$ should be in $I$. But now the rules of the form (\ref{eq:ruleacc}) imply that $accept(m_i,w_j)$ should be in $I$, contradicting step 2 since $accept(m_i,w_{j'})$ and $accept(m_{i'},w_j)$ are already in $I$.
\end{enumerate}
\end{proof}

\begin{proposition} \label{pr:SMPASPindifssas}
Let $(S_M,S_W)$ be an instance of the SMP with unacceptability and ties, and let $\PP$ be the corresponding ASP program. If $\{(x_{1},y_{1})$, $\hdots$, $(x_{k},y_{k})\}$ is a weakly stable matching for $(S_M,S_W)$ then $\PP$ has the following answer set $I$:
\allowdisplaybreaks
\begin{align*}
&\{manpropose(x_{i},y)\,|\, i \in \{1,\hdots,k\}, x_{i}\in M, y <_M^{x_i} y_{i})\} \\
\cup &\{womanpropose(x,y_{i})\,|\,i \in \{1,\hdots,k\}, y_{i}\in W, x <_W^{y_i} x_i\} \\
\cup &\{accept(x_{i},y_{i}) \,|\, i \in \{1,\hdots,k\}\} \\
\cup &\{manpropose(x_i,y_i) \,|\,  i \in \{1,\hdots,k\}, x_i \neq y_i\} \\
\cup&\{womanpropose(x_i,y_i) \,|\,  i \in \{1,\hdots,k\}, x_i \neq y_i \} 
\end{align*} 
\end{proposition}
\begin{proof}
Suppose we have a stable set of marriages $S=\{(x_1,y_1),\hdots,(x_k,y_k)\}$, implying that every $y_i$ is an acceptable partner of $x_i$ and the other way around. The rules of the form (\ref{eq:ruleacc}) do not alter when forming the reduct, but the other rules do as those contain naf-literals. Notice first that the stability of $S$ implies that there cannot be an unmarried couple $(m,w)$, with $m$ a man and $w$ a woman, such that $manpropose(m,w)$ is in $I$ and $womanpropose(m,w)$ is in $I$. By definition of $I$ this would mean that they both strictly prefer each other to their current partner in $S$. This means they would form a blocking pair, but since $S$ was stable, that is impossible. So the rules of the form (\ref{eq:ruleacc}) will be applied exactly for married couples $(m_i,w_j)$, since by definition of $I$ $manpropose(m_i,w_j)$ and $womanpropose(m_i,w_j)$ are both in $I$ under these conditions. For other cases the rule will also be fulfilled since the body will be false. This reasoning implies that the unique minimal model of the reduct w.r.t.\ $I$ should indeed contain $accept(m_i,w_j)$ for every married couple $(m_i,w_j)$ in $S$.
Since $S$ is a stable set of marriages, every person is either married or single. If a man $m_i$ is single, there will be no other literal of the form $accept(m_i,.)$ in $I$, so rule (\ref{eq:msingleindif}) will reduce to a fact $accept(m_i,m_i) \la$, which is obviously fulfilled by $I$. Similarly if a woman $w_j$ is single. Any other rule of the form (\ref{eq:msingleindif}) or (\ref{eq:wsingleindif}) is deleted because $m_i$ or $w_j$ is not single in that case and thus there is some literal of the form $accept(m_i,w)$ for some woman $w$ and some literal of the form $accept(m,w_j)$ for some man $m$ in $I$, falsifying the body of the rules. If $m_i$ is single, then $accept(m_i,m_i)$ is in $I$ and this is the only literal of the form $accept(m_i,.)$ in $I$, so the rules of the form (\ref{eq:mpropindif}) will all be reduced to facts. The rule heads of these facts should be in the minimal model of the reduct and are indeed in $I$ since the women $w$ for which $manpropose(m_i,w)$ is in $I$ are exactly those who are strictly preferred to staying single. The rules of the form (\ref{eq:mpropindif}) for women $w_j$ in $neutral_M^i$ will all be deleted in this case, because $accept(m_i,m_i)$ is in $I$. If man $m_i$ is married to a certain woman $w_j$ in the stable set $S$ then the rules of the form (\ref{eq:mpropindif}) will reduce to facts of the form $manpropose(m_i,w) \la$ for every woman $w$ who is strictly preferred to $w_j$ and will be deleted for every other woman appearing in the head, because those rules will contain $not\, accept(m_i,w_j)$ in the body. Again $I$ contains these facts by definition, as the minimal model of the reduct should. We can use an analogous reasoning for the women. So the presence of the literals of the form $manpropose(.,.)$, $womanpropose(.,.)$ and $accept(.,.)$ in $I$ is required in the unique minimal model of the reduct w.r.t.\ $I$. 
We have proved that every literal in $I$ should be the minimal model of the reduct and that every rule of the reduct is fulfilled by $I$, implying that $I$ is an answer set of $\PP$. 
\end{proof}

In~\cite{SMPMan02} it is shown that the decision problem `is the pair $(m,w)$ stable?' for a given SMP instance with unacceptablity and ties is an NP-complete problem, even in the absence of unacceptability. A pair $(m,w)$ is \textit{stable} if there exists a stable set that contains $(m,w)$. It is straightforward to see that we can reformulate this decision problem as `does there exist an answer set  of the induced normal ASP program $\PP$ which contains the literal $accept(m,w)$?' (i.e.\ brave reasoning), which is known to be an NP-complete problem~\cite{ASPBa03}. So our model forms a suitable framework for these kind of decision problems concerning the SMP. 

\section{Selecting Preferred Stable Sets} \label{sec:oss}

\subsection{Notions of Optimality of Stable Sets}
When several stable matchings can be found for an instance of the SMP, some may be more interesting than others. The stable set found by the G-S algorithm is \textit{M-optimal}~\cite{SMPRoth90}, i.e.\ every man likes this set at least as well as any other stable set. 
Exchanging the roles of men and women in the G-S algorithm yields a \textit{W-optimal} stable set~\cite{SMPGale62}, optimal from the point of view of the women. 

While some applications may require us to favour either the men or the women, in others it makes more sense to treat both parties equally. To formalize some commonly considered notions of fairness and optimality w.r.t.\ the SMP, we define the cost $c_x(S)$ of a stable set $S$ to an individual $x$, where $c_x(S)=k$ if $x$ has been matched with his or her $k^{th}$ preferred partner. More precisely, for $x=m_i$ a man, we define $c_{m_i}(S) = | \{ z : z <_M^{m_i} y \} |+1$ where $y$ is the partner of $x$ in $S$; for $x=w_j$ a woman, $c_x$ is defined analogously. So in case of ties we assign the same list position to equally preferred partners, as illustrated in Example \ref{ex:cost}.
\begin{example}\label{ex:cost}
Let $x=m_1$ be a man with preference list $\sigma_M^1 = (\{1\},\{2, 3\},\{4\})$ then $w_1$ as partner of $x$ in some set of marriages $S$ would yield $c_x(S)=1$, $w_2$ and $w_3$ yield $c_x(S)=2$ and $w_4$ yields $c_x(S)=4$. If $m_1$ would be single in $S$, then the cost $c_x(S)$ is $4$, since $m_1$ prefers women $w_1,w_2$ and $w_3$ to being single, but is indifferent between being paired to $w_4$ or staying single.
\end{example}
\begin{definition}\label{def:optss}
For $S$ a set of marriages,
\begin{itemize}
\item the sex-equalness cost is defined as $c_{sexeq}(S)=|\sum_{x \in M}{c_x(S)} - \sum_{x \in W}{c_x(S)}|$,
\item the egalitarian cost is defined as $c_{weight}(S)=\sum_{x \in M \cup W}{c_x(S)}$, 
\item the regret cost is defined as $c_{regret}(S)=\max_{x \in M \cup W}{c_x(S)}$, and
\item the cardinality cost is defined as $c_{singles}(S)=|\{z : (z,z) \in S\}|$.
\end{itemize}
$S$ is a sex-equal stable set iff $S$ is a stable set with minimal sex-equalness cost. Similarly, $S$ is an egalitarian (resp.\ minmum regret, maximum cardinality) stable set iff $S$ is a stable set with minimal egalitarian (resp.\ regret or cardinality) cost. 
\end{definition}
A \textit{sex-equal stable set} assigns an equal importance to the preferences of the men and women. An \textit{egalitarian stable set} is a stable set in which the preferences of every individual are considered to be equally important. In~\cite{SMPXu11} the use of an egalitarian stable set is proposed to optimally match virtual machines (VM) to servers in order to improve cloud computing by equalizing the importance of migration overhead in the data center network and VM migration performance.
A \textit{minimum regret stable set} is optimal for the person who is worst off. A \textit{maximal or minimal cardinality stable set} is a stable set with resp.\ as few or as many singles as possible. Examples of practical applications include an efficient kidney exchange program~\cite{SMPRoth05} and the National Resident Matching Program\footnote[3]{www.nrmp.org}~\cite{SMPMan02}. Maximizing cardinality garantuees that as many donors as possible will get a compatible donor and as many medical graduates as possible will get a position.

Table \ref{tab:compl} presents an overview of known complexity results\footnote[4]{Throughout this paper we assume that P $\neq$ NP.} concerning finding an optimal stable set. Typically the presence of ties leads to an increase of complexity. Manlove et al. \cite{SMPMan99,SMPMan02} proved that the problem of finding a maximum (or minimum) cardinality stable set for a given instance of the SMP with unacceptability and ties is NP-hard. Using this result, the problem of finding an egalitarian or minimum regret stable matching for a given SMP instance with ties is proved to be NP-hard~\cite{SMPMan02}, even if the ties occur on one side only and each tie is of length 2 (i.e.\ each set in a preference list has size at most 2). 
If there are no ties, the problem of finding an egalitarian or minimum regret stable set is solvable in polynomial time~\cite{SMPIrv87,SMPGus87}. Since all stable sets consist of $n$ couples in the classical SMP, the G-S algorithm trivially finds a maximum (or minimum) cardinality~\cite{SMPGale62}. For an SMP instance with unacceptability the number of couples in a stable set is constant~\cite{SMPGale85}, so finding a maximum cardinality stable set reduces to finding a stable set, which is known to be solvable in polynomial time. Surprisingly, finding a sex-equal stable set for a classical SMP instance is NP-hard~\cite{SMPKat93}, even if the preference lists are bound in length by 3~\cite{SMPMc12}.
\begin{table}[h!]
\caption{Literature complexity results for finding an optimal stable set}
\begin{center}
\begin{tabular}{c|cccc}
& sex-equal & egalitarian & min.\ regret & max.\ card.\ \\
\hline
SMP & NP-hard~\cite{SMPKat93} & P ($O(n^4)$~\cite{SMPIrv87}) & P ($O(n^2)$~\cite{SMPGus87}) & P ($O(n^2)$~\cite{SMPGale62})\\
SMP + unacc & NP-hard~\cite{SMPMc12} & & & P~\cite{SMPGale85}\\
SMP + ties & & NP-hard~\cite{SMPMan02} & NP-hard~\cite{SMPMan02} &\\
SMP + \{unacc,ties\} & & & & NP-hard~\cite{SMPMan99,SMPMan02}\\
\end{tabular}
\end{center}
\label{tab:compl}
\end{table}

Between brackets we mention in Table \ref{tab:compl} the complexity of an algorithm that finds an optimal stable set if one exists, in function of the number of men $n$. To the best of our knowledge, the only exact algorithm tackling an NP-hard problem from Table \ref{tab:compl} finds a sex-equal stable set for an SMP instance in which the strict preference lists of men and/or women are bounded in length by a constant~\cite{SMPMc12}. To the best of our knowledge, no exact implementations exist to find an optimal stable set for an SMP instance with ties, regardless of the presence of unacceptability and regardless which notion of optimality from Table \ref{tab:compl} is used. Our approach yields an exact implementation of all problems mentioned in Table \ref{tab:compl}.

\subsection{Finding Optimal Stable Sets using Disjunctive ASP}
As we discuss next, we can extend our ASP encoding of the SMP such that the optimal stable sets correspond to the answer sets of an associated ASP program. In particular, we use a saturation technique~\cite{ASPEit97,ASPBa03} to filter non-optimal answer sets. 
Intuitively, the idea is to create a program with 3 components: (i) a first part describing the solution candidates, (ii) a second part also describing the solution candidates since comparison of solutions requires multiple solution candidates within the same answer set whereas the first part in itself produces one solution per answer set, (iii) a third part comparing the solutions described in the first two parts and selecting the preferred solutions by saturation.
It is known that the presence of negation-as-failure can cause problems when applying saturation. Therefore, we use a SAT encoding~\cite{ASPJan04} of the ASP program in Definition \ref{def:aspsmpindif} and define a disjunctive naf-free ASP program in Definition \ref{def:disjaspsmpindif} which selects particular models of the SAT problem. 
Notice that our original normal program is absolutely tight, i.e.\ there is no finite sequence $l_1,l_2,\hdots$ of literals such that for every $i$ there is a program rule for which $l_{i+1}$ is a positive body literal and $l_i$ is in the head~\cite{ASPEr03}. 
We use the completion and a translation of our ASP program to SAT to derive Definition \ref{def:disjaspsmpindif}. The completion of a normal ASP program is a set of propositional formulas. For every atom $a$ with $a \la body_i$ ($i \in \{1,\hdots,k\}$) all the program rules with head $a$, the propositional formula $a \equiv body'_1 \vee \hdots \vee body'_k$ is in the completion of that program. If an atoms $a$ of the program does not occur in any rule head, than $a \equiv \, \perp$ is in the completion of the program. Similarly the completion of the program contains the propositional formula $\perp \,\equiv body'_1 \vee \hdots \vee body'_l$ with with $\la body_i$ ($i \in \{1,\hdots,l\}$) all the program constraints. For every $i$, $body'_i$ is the conjunction of literals derived from $body_i$ by replacing every occurence of `$not$' with `$\neg$'. Because our program is absolutely tight, we know that the completion will correspond to it~\cite{ASPEr03}. 
Applied to the induced normal ASP program in Definition \ref{def:aspsmpindif}, the completion becomes: 
\allowdisplaybreaks
\begin{align*}
&\{accept(m_i,w_j) \equiv manpropose(m_i,w_j) \wedge womanpropose(m_i,w_j) \,|\, i \in\{1,\hdots,n\},j\in\{1,\hdots,p\} \}\\
\cup &\{accept(m_i,m_i) \equiv \bigwedge_{k \in acceptable_M^i}{ \neg accept(m_i,w_k)} \,|\, i \in\{1,\hdots,n\}\}\\
\cup &\{accept(w_j,w_j) \equiv \bigwedge_{k \in acceptable_W^j}{ \neg accept(m_k,w_j)} \,|\, j \in\{1,\hdots,p\}\}\\
\cup &\{manpropose(m_i,w_j) \equiv \bigwedge_{x \leq_M^{m_i} w_j, x \neq w_j}{\neg accept(m_i,x)} \,|\, i \in\{1,\hdots,n\}, j\in acceptable_M^i\}\\
\cup &\{womanpropose(m_i,w_j) \equiv \bigwedge_{x \leq_W^{w_j} m_i, x \neq m_i}{\neg accept(x,w_j)} \,|\, j\in\{1,\hdots,p\}, i \in acceptable_W^j\}\\
\cup &\{manpropose(m_i,w_j) \equiv \,\perp \,|\, i \in \{1,\hdots,n\}, j \in unacceptable_M^i \} \\
\cup &\{womanpropose(m_i,w_j) \equiv \,\perp \,|\, j \in \{1,\hdots,p\}, i \in unacceptable_W^j \}
\end{align*}
Using the formulas of the completion corresponding to the normal ASP program in Definition \ref{def:aspsmpindif}, we can define a corresponding disjunctive ASP program without negation-as-failure. Lemma \ref{lem:disj} follows form the fact that the completion corresponds to the original program~\cite{ASPEr03}.
\begin{definition}[Induced disj.\ naf-free ASP program] \label{def:disjaspsmpindif}
The disjunctive naf-free ASP program $\PP_{disj}$ induced by an SMP instance $(S_M,S_W)$ with unacceptability and ties contains the following rules for $i\in\{1,\hdots,n\},j\in\{1,\hdots,p\}$:
\allowdisplaybreaks
\begin{align*}
\neg accept(m_i,w_j) \vee manpropose(m_i,w_j) &\la \\
\neg accept(m_i,w_j) \vee womanpropose(m_i,w_j) &\la\\
accept(m_i,w_j) \vee \neg manpropose(m_i,w_j) \vee \neg womanpropose(m_i,w_j) &\la
\end{align*}
For every $i\in\{1,\hdots,n\}$, $l\in unacceptable_M^i$, $j \in acceptable_M^i$, $x \leq_M^{m_i} w_j, x\neq w_j$ $\PP_{disj}$  contains:
\begin{align*}
\bigvee_{k \in acceptable_M^i} accept(m_i,w_k) \vee accept(m_i,m_i) &\la\\
\neg accept(m_i,m_i) \vee \neg accept(m_i,w_j) &\la \\
\neg manpropose(m_i,w_j) \vee \neg accept(m_i,x) &\la \\
\bigvee_{x \leq_M^{m_i} w_j, x \neq w_j} accept(m_i,x) \vee manpropose(m_i,w_j) &\la\\
\neg manpropose(m_i,w_l) &\la
\end{align*}
and symmetrical for $j \in \{1,\hdots,p\}$ and $womanpropose$.
\end{definition}

\begin{lemma} \label{lem:disj}
Let $\PP$ be the normal ASP program from Definition \ref{def:aspsmpindif} and $\PP_{disj}$ the disjunctive ASP program from Definition \ref{def:disjaspsmpindif}. It holds that for any answer set $I$ of $\PP$ there exists an answer set $I_{disj}$ of $\PP_{disj}$ such that the atoms of $I$ and $I_{disj}$ coincide. Conversely for any answer set $I_{disj}$ of $\PP_{disj}$ there exists an answer set $I$ of $\PP$ such that the atoms of $I$ and $I_{disj}$ coincide.
\end{lemma}

\subsection{ASP Program to Select Optimal Solutions} \label{sec:defOSSASP}
Let $(S_M,S_W)$ be an SMP instance with unacceptability and ties, with $S_M = \{\sigma_M^1,\hdots,\sigma_M^n\}$ and $S_W = \{\sigma_W^1,\hdots,\sigma_W^p\}$, and let $\PP_{norm}$ be the induced normal ASP program from Definition \ref{def:aspsmpindif}. Our technique for extending this program to a program that can respectively optimize for the sex-equalness, egalitarian, minimum regret and maximum cardinality criterion is in each case very similar. We start by explaining it for the case of sex-equalness. Our first step is to add a set of rules that compute the sex-equalness cost of a set of marriages. For every man $m_i$ and every woman $w_j$ such that $j \in \sigma^i_M(k)$ we use the following rule to determine the cost for $m_i$ if $w_j$ would be his partner:
\begin{align}
mancost(i,k) &\la accept(m_i,w_j) \label{eq:costm}
\end{align}
and similarly for every $w_j$ and every $m_i$ such that $i \in \sigma^j_W(k)$:
\begin{align}
womancost(j,k) &\la accept(m_i,w_j) \label{eq:costw}
\end{align}
We also use the following rules with $i$ ranging from $1$ to $n$ and $j$ from $1$ to $p$:
\allowdisplaybreaks
\begin{align}
mancost(i,|\sigma^i_M|) &\la accept(m_i,m_i) \label{eq:costmm}\\
womancost(j,|\sigma^j_W|) &\la accept(w_j,w_j) \label{eq:costww}\\
manweight(Z) &\la \#sum\{B,A : mancost(A,B)\}=Z, \#int(Z) \label{eq:weightm}\\
womanweight(Z) &\la \#sum\{B,A : womancost(A,B)\}=Z, \#int(Z)  \label{eq:weightw}\\
sexeq(Z) \la man&weight(X), womanweight(Y), Z=X-Y \nonumber\\
sexeq(Z) \la man&weight(X), womanweight(Y), Z=Y-X \label{eq:sexeq}
\end{align}
Rules (\ref{eq:costmm}) and (\ref{eq:costww}) state staying single leads to the highest cost. Rule (\ref{eq:weightm}) determines the sum of the male costs\footnotemark[5] and similarly (\ref{eq:weightw}) determines the sum of the female costs. According to Definition \ref{def:optss} the absolute difference of these values yields the sex-equalness cost, as determined by rules (\ref{eq:sexeq}). Since numeric variables are restricted to positive integers in DLV, we omit conditions as `$X\geq Y$' or `$X < Y$'. The program $\PP_{norm}$ extended with rules (\ref{eq:costm}) -- (\ref{eq:sexeq}) is denoted $\PP^{sexeq}_{ext}$.
\footnotetext[5]{$\#sum$, $\#max$, $\#int$ and $\#count$ are DLV aggregate functions. The `$A$' mentioned as variable in $\#sum$ indicates that a cost must be included for every person (otherwise the cost is included only once when persons have the same cost).}
We construct a program $\PP_{sexeq}$, composed by subprograms, that selects optimal solutions. Let $\PP'_{disj}$ be the disjunctive naf-free ASP program, induced by the same SMP instance, in which a prime symbol is added to all literal names (e.g.\ $accept$ becomes $accept'$). Define a new program $\PP'^{sexeq}_{ext}$ with all the rules of $\PP'_{disj}$ in which every occurrence of $\neg atom$ is changed into $natom$ for every atom $atom$, i.e.\ replace all negation symbols by a prefix `$n$'. For every occurring atom $atom$ in $\PP'^{sexeq}_{ext}$, add the following rule to exclude non-consistent solutions\footnote[6]{For instance, $sat \la accept'(m_1,w_1), naccept'(m_1,w_1)$}:
\begin{align}
sat &\la atom, natom \label{eq:geenneg}
\end{align}
Finally add rules (\ref{eq:costm}) -- (\ref{eq:sexeq}) with prime symbols to the literal names to $\PP'^{sexeq}_{ext}$ but replace rule (\ref{eq:weightm}) and rule (\ref{eq:weightw}) by:
\begin{align}
mansum(n,X) &\la mancost(n,X) \nonumber\\
mansum(J,Z) &\la mansum(I,X), mancost(J,Y), Z=X+Y, \#succ(J,I) \nonumber\\ 
manweight(Z) &\la mansum(1,Z) \nonumber\\
womansum(p,X) &\la womancost(p,X) \nonumber\\
womansum(J,Z) &\la womansum(I,X), womancost(J,Y), Z=X+Y, \#succ(J,I) \nonumber\\ 
womanweight(Z) &\la womansum(1,Z) \label{eq:weightsucc}
\end{align}
The DLV aggregate function $\#succ(J,I)$ is true whenever $J+1=I$. The reason we replace the rules with the aggregate function $\#max$ by these rules is to make sure the saturation happens correct. When saturation is used, the DLV aggregate function $\#max$, $\#sum$ and $\#count$ would not yield the right criteriumvalues. Moreover, DLV does not accept these aggregate function in saturation because of the cyclic dependency of literals  within the aggragate functions created by the rules for saturation. These adjusted rules, however, will do the job  because of the successive way they compute the criteriumvalues. This becomes more clear in the proof of Proposition \ref{pr:optssASP}. We define the ASP program $\PP_{sexeq}$ as the union of $\PP^{sexeq}_{ext}$, $\PP'^{sexeq}_{ext}$ and $\PP_{sat}$. The ASP program $\PP_{sat}$ contains the following rules to select minimal solutions based on sex-equalness:
\allowdisplaybreaks
\begin{align} \label{eq:satcrit}
sat  &\la sexeq(X), sexeq'(Y), X \leq Y \\
&\la not \,  sat \label{eq:sat}\\
mancost'(X,Y) &\la  sat, manargcost'_1(X), manargcost'_2(Y) \nonumber \\
womancost'(X,Y) &\la  sat, womanargcost'_1(X), womanargcost'_2(Y)  \label{eq:cost}\\
manpropose'(X,Y) &\la  sat, man(X), woman(Y)  \nonumber\\
womanpropose'(X,Y) &\la  sat, man(X), woman(Y) \nonumber \\
accept'(X,X) & \la  sat, man(X) \nonumber\\
accept'(X,X) & \la  sat, woman(X) \nonumber \\
accept'(X,Y) &\la  sat, man(X), woman(Y) \label{eq:satur}
\end{align}
and analogous to (\ref{eq:satur}) a set of rules with prefix `$n$' for the head predicates.
Finally we add the facts\footnote[7]{The rule $manargcost'_1(1..n) \la$ is DLV-syntax for the $n$ facts $manargcost'_1(1) \la, \hdots, manargcost'_1(n) \la$.} $manargcost'_1(1..n) \la$, $manargcost'_2(1..(p+1)) \la$, $womanargcost'_1(1..p) \la$, $womanargcost'_2(1..(n+1)) \la$, $man(x) \la$ for every man $x$ and $woman(x) \la$ for every woman $x$ to $\PP_{sat}$.
Intuitively the rules of $\PP_{sat}$ express the key idea of saturation. First every answer set is forced to contain the atom $sat$ by rule (\ref{eq:sat}). Then the rules (\ref{eq:cost}) -- (\ref{eq:satur}) and the facts make sure that any answer set should contain all possible literals with a prime symbol that occur in $\PP_{sexeq}$. Rule (\ref{eq:satcrit}) will establish that only optimal solutions will correspond to minimal models and thus lead to answer sets. For any non-optimal solution, the corresponding interpretation containing $sat$ will never be a minimal model of the reduct. It is formally proved in Proposition \ref{pr:optssASP} below that $\PP_{sexeq}$ produces exactly the stable matchings with minimal sex-equalness cost.

Furthermore, only small adjustments to $\PP_{sexeq}$ are needed to create programs $\PP_{weight}$, $\PP_{regret}$, and $\PP_{singles}$ that resp.\ produce egalitarian, minimum regret and maximum cardinality stable sets. Indeed, the ASP program $\PP_{weight}$ can easily be defined as $\PP_{sexeq}$ in which the predicates $sexeq$ and $sexeq'$ are resp.\ replaced by $weight$ and $weight'$ and the rules (\ref{eq:sexeq}) are replaced by (\ref{eq:weight}), determining the egalitarian cost of Definition \ref{def:optss} as the sum of the male and female costs:
\begin{align}
weight(Z) &\la manweight(X),womanweight(Y),Z=X+Y\label{eq:weight}
\end{align}

Similarly the ASP program $\PP_{regret}$ is defined as $\PP_{sexeq}$ in which the predicates $sexeq$ and $sexeq'$ are resp.\ replaced by $regret$ and $regret'$ and rules (\ref{eq:weightm}) -- (\ref{eq:sexeq}) are replaced by the following rules:
\begin{align}
manregret(Z) &\la \#max\{B : mancost(A,B)\}=Z, \#int(Z) \label{eq:regretm}\\
womanregret(Z) &\la \#max\{B : womancost(A,B)\}=Z, \#int(Z) \label{eq:regretw}\\
regret(X) &\la manregret(X), womanregret(Y), X>Y \nonumber\\
regret(Y) &\la manregret(X), womanregret(Y), X \leq Y \label{eq:regret}
\end{align} 
Rule (\ref{eq:regretm}) determines the regret cost but only for the men. Similarly (\ref{eq:regretw}) determines the regret cost for the women. The regret cost as defined in Definition \ref{def:optss} is the maximum of these two values, determined by the rules in (\ref{eq:regret}).
Again we adjust rules (\ref{eq:regretm}) and (\ref{eq:regretw}) for the program part $\PP'^{regret}_{ext}$ by replacing them with a successively computing variant:
\begin{align}
manmax(n,X) &\la mancost(n,X)\nonumber\\
manmax(J,X) &\la manmax(I,X), mancost(J,Y), X \geq Y, \#succ(J,I)\nonumber\\
manmax(J,Y) &\la manmax(I,X), mancost(J,Y), X < Y, \#succ(J,I)\nonumber\\
manregret(Z) &\la manmax(1,Z)\nonumber\\
womanmax(p,X) &\la womancost(p,X)\nonumber\\
womanmax(J,X) &\la womanmax(I,X), womancost(J,Y), X \geq Y, \#succ(J,I)\nonumber\\
womanmax(J,Y) &\la womanmax(I,X), womancost(J,Y), X < Y, \#succ(J,I)\nonumber\\
womanregret(Z) &\la womanmax(1,Z) \label{eq:regretsucc}
\end{align} 

Finally we define the ASP program $\PP_{singles}$ as $\PP_{sexeq}$ in which the predicates $sexeq$ and $sexeq'$ are resp.\ replaced by $singles$ and $singles'$. Furthermore we replace rules (\ref{eq:costm}) -- (\ref{eq:sexeq}) by (\ref{eq:card}), determining the number of singles:
\begin{align}
singles(Z) \la \#count\{B : accept(B,B)\} = Z, \#int(Z) \label{eq:card}
\end{align}
This time we adjust rule (\ref{eq:card}) for the program part $\PP'^{singles}_{ext}$ as follows:
\begin{align}
single(p+i,1) &\la accept(m_i,m_i)\nonumber\\
single(p+i,0) &\la naccept(m_i,m_i)\nonumber\\
single(j,1) &\la accept(w_j,w_j)\nonumber\\
single(j,0) &\la naccept(w_j,w_j)\nonumber\\
singlesum(n+p,X) &\la single(n+p,X)\nonumber\\
singlesum(J,Z) &\la singlesum(I,X), single(J,Y), Z=X+Y, \#succ(J,I)\nonumber\\
singles(Z) &\la sat, singlesum(1,Z) \label{eq:singlessucc}
\end{align}

\begin{example}
We reconsider Example \ref{ex:smpasp}. This SMP instance had 3 stable sets of marriages:
\begin{itemize}
\item $S_1 = \{accept(m_1,w_3), accept(m_2,w_1), accept(w_2,w_2)\}$,
\item $S_2 = \{accept(m_1,w_2), accept(m_2,w_1), accept(w_3,w_3)\}$,
\item $S_3 = \{accept(m_1,w_1), accept(m_2,m_2), accept(w_2,w_2), accept(w_3,w_3)\}$.
\end{itemize}
It is easy to compute the respective regret costs as $c_{regret}(S_1) = 2$ and $c_{regret}(S_2) = c_{regret}(S_3) = 3$. The corresponding program selecting this minimum regret stable set is the program consisting of the rules in Example \ref{ex:smpasp} in addition with:
\allowdisplaybreaks
\begin{align*}
mancost(1,1) &\la accept(m_1,w_1)\\
mancost(1,2) &\la accept(m_1,w_2)\\
mancost(1,2) &\la accept(m_1,w_3)\\
mancost(1,3) &\la accept(m_1,m_1)\\
mancost(2,2) &\la accept(m_2,w_1)\\
mancost(2,1) &\la accept(m_2,w_2)\\
mancost(2,2) &\la accept(m_2,m_2)\\
womancost(1,1) &\la accept(m_1,w_1)\\
womancost(1,1) &\la accept(m_2,w_1)\\
womancost(1,2) &\la accept(w_1,w_1)\\
womancost(2,1) &\la accept(m_1,w_2)\\
womancost(2,2) &\la accept(w_2,w_2)\\
womancost(3,2) &\la accept(m_1,w_3)\\
womancost(3,1) &\la accept(m_2,w_3)\\
womancost(3,3) &\la accept(w_3,w_3)\\
manregret(Z) &\la \#max\{B: mancost(A,B)\}=Z, \#int(Z)\\
womanregret(Z) &\la \#max\{B: womancost(A,B)\}=Z, \#int(Z)\\
regret(X) &\la manregret(X), womanregret(Y), X>Y\\
regret(Y) &\la manregret(X), womanregret(Y), X<=Y
\end{align*}
\begin{align*}
naccept'(M,W) \vee manpropose'(M,W) &\la man(M), woman(W)\\
naccept'(M,W) \vee womanpropose'(M,W) &\la man(M), woman(W)\\
accept'(M,W) \vee nmanpropose'(M,W) \vee nwomanpropose'(M,W) &\la man(M), woman(W)\\
accept'(m_1,w_1) \vee accept'(m_1,w_2) \vee accept'(m_1,w_3) \vee accept'(m_1,m_1) &\la \\
accept'(m_2,w_1) \vee accept'(m_2,w_2) \vee accept'(m_2,m_2) &\la \\
naccept'(m_1,m_1) \vee naccept'(m_1,w_1) &\la \\
naccept'(m_1,m_1) \vee naccept'(m_1,w_2) &\la \\
naccept'(m_1,m_1) \vee naccept'(m_1,w_3) &\la \\
naccept'(m_2,m_2) \vee naccept'(m_2,w_1) &\la \\
naccept'(m_2,m_2) \vee naccept'(m_2,w_2) &\la \\
accept'(m_1,w_1) \vee accept'(m_2,w_1) \vee accept'(w_1,w_1) &\la \\
accept'(m_1,w_2) \vee accept'(w_2,w_2) &\la \\
accept'(m_1,w_3) \vee accept'(m_2,w_3) \vee accept'(w_3,w_3) &\la \\
naccept'(w_1,w_1) \vee naccept'(m_1,w_1) &\la \\
naccept'(w_1,w_1) \vee naccept'(m_2,w_1) &\la \\
naccept'(w_2,w_2) \vee naccept'(m_1,w_2) &\la \\
naccept'(w_3,w_3) \vee naccept'(m_1,w_3) &\la \\
naccept'(w_3,w_3) \vee naccept'(m_2,w_3) &\la \\
nmanpropose'(m_1,w_2) \vee naccept'(m_1,w_1) &\la \\
nmanpropose'(m_1,w_2) \vee naccept'(m_1,w_3) &\la \\
nmanpropose'(m_1,w_3) \vee naccept'(m_1,w_1) &\la \\
nmanpropose'(m_1,w_3) \vee naccept'(m_1,w_2) &\la \\
manpropose'(m_1,w_1) &\la \\
accept'(m_1,w_1) \vee accept'(m_1,w_3) \vee manpropose'(m_1,w_2) &\la \\
accept'(m_1,w_1) \vee accept'(m_1,w_2) \vee manpropose'(m_1,w_3) &\la \\
nmanpropose'(m_2,w_1) \vee naccept'(m_2,w_2) &\la \\
nmanpropose'(m_2,w_1) \vee naccept'(m_2,m_2) &\la \\
manpropose'(m_2,w_2) &\la \\
accept'(m_2,w_2) \vee accept'(m_2,m_2) \vee manpropose'(m_2,w_1) &\la \\
nwomanpropose'(m_1,w_1) \vee naccept'(m_2,w_1) &\la \\
nwomanpropose'(m_2,w_1) \vee naccept'(m_1,w_1) &\la \\
accept'(m_1,w_1) \vee womanpropose'(m_2,w_1) &\la \\
accept'(m_2,w_1) \vee womanpropose'(m_1,w_1) &\la \\
womanpropose'(m_1,w_2) &\la \\
nwomanpropose'(m_1,w_3) \vee naccept'(m_2,w_3) &\la \\
womanpropose'(m_2,w_3) &\la \\
accept'(m_2,w_3) \vee womanpropose'(m_1,w_3) &\la \\
nmanpropose'(m_2,w_3) &\la\\
nwomanpropose'(m_2,w_2) &\la 
\end{align*}
\begin{align*}
sat &\la manpropose'(X,Y), nmanpropose'(X,Y), man(X), woman(Y)\\
sat &\la womanpropose'(X,Y), nwomanpropose'(X,Y), man(X), woman(Y)\\
sat &\la accept'(X,Y), naccept'(X,Y), man(X), woman(Y)\\
sat &\la accept'(X,X), naccept'(X,X), man(X)\\
sat &\la accept'(X,X), naccept'(X,X), woman(X)\\
mancost'(1,1) &\la accept'(m_1,w_1)\\
mancost'(1,2) &\la accept'(m_1,w_2)\\
mancost'(1,2) &\la accept'(m_1,w_3)\\
mancost'(1,3) &\la accept'(m_1,m_1)\\
mancost'(2,2) &\la accept'(m_2,w_1)\\
mancost'(2,1) &\la accept'(m_2,w_2)\\
mancost'(2,2) &\la accept'(m_2,m_2)\\
womancost'(1,1) &\la accept'(m_1,w_1)\\
womancost'(1,1) &\la accept'(m_2,w_1)\\
womancost'(1,2) &\la accept'(w_1,w_1)\\
womancost'(2,1) &\la accept'(m_1,w_2)\\
womancost'(2,2) &\la accept'(w_2,w_2)\\
womancost'(3,2) &\la accept'(m_1,w_3)\\
womancost'(3,1) &\la accept'(m_2,w_3)\\
womancost'(3,3) &\la accept'(w_3,w_3)\\
manmax'(2,X) &\la mancost'(2,X)\\
manmax'(J,X) &\la manmax'(I,X), mancost'(J,Y), X>=Y, \#succ(J,I)\\
manmax'(J,X) &\la manmax'(I,X), mancost'(J,Y), X>=Y, \#succ(J,I)\\
manregret'(Z) &\la manmax'(1,Z)\\
womanmax'(2,X) &\la womancost'(2,X)\\
womanmax'(J,X) &\la womanmax'(I,X), womancost'(J,Y), X>=Y, \#succ(J,I)\\
womanmax'(J,X) &\la womanmax'(I,X), womancost'(J,Y), X>=Y, \#succ(J,I)\\
womanregret'(Z) &\la womanmax'(1,Z)\\
regret'(X) &\la manregret'(X), womanregret'(Y), X>Y\\
regret'(Y) &\la manregret'(X), womanregret'(Y), X<=Y\\
sat &\la regret(X), regret'(Y), X<=Y\\
&\la not\, sat\\
manargcost_1'(1..2) &\la \\
manargcost_2'(1..4) &\la  \\
womanargcost_1'(1..3) &\la \\
womanargcost_1'(1..3) &\la \\
mancost'(X,Y) &\la sat, manargcost_1'(X),manargcost_2'(Y)\\
womancost'(X,Y) &\la sat, womanargcost_1'(X),womanargcost_2'(Y)\\
manpropose'(X,Y) &\la sat, man(X), woman(Y)\\
nmanpropose'(X,Y) &\la sat, man(X), woman(Y)\\
womanpropose'(X,Y) &\la sat, man(X), woman(Y)\\
nwomanpropose'(X,Y) &\la sat, man(X), woman(Y)\\
accept'(X,Y) &\la sat, man(X), woman(Y)\\
accept'(X,X) &\la sat, man(X)\\
accept'(X,X) &\la sat, woman(X)\\
naccept'(X,Y) &\la sat, man(X), woman(Y)\\
naccept'(X,X) &\la sat, man(X)\\
naccept'(X,X) &\la sat, woman(X)
\end{align*}
Letting DLV compute the unique answer set of this disjunctive ASP program and filtering it to the literals $accept$ and $regret$, yields $\{accept(m_2,w_1)$, $accept(m_1,w_3)$, $accept(w_2,w_2)$, $regret(2)\}$, corresponding exactly to the minimum regret stable set of the SMP instance and the corresponding regret cost.
\end{example}

\begin{proposition} \label{pr:optssASP}
Let the criterion $crit$ be an element of $\{sexeq$, $weight$, $regret$, $singles\}$. For every answer set $I$ of the program $\PP_{crit}$ induced by an SMP instance with unacceptability and ties the set $S_I = \{(m,w) \,|\,$ $accept(m$, $w)$ $\in$ $I\}$ forms an optimal stable set of marriages w.r.t.\ criterion $crit$ and the optimal criterion value is given by the unique value $v_I$ for which $crit(v_I) \in I$. Conversely for every optimal stable set $S = \{(x_1,y_1), \hdots, (x_k,y_k)\}$ with optimal criterion value $v$ there exists an answer set $I$ of $\PP_{crit}$ such that $\{(x,y) \,|\,$ $accept(x,y)$ $\in I\} = \{(x_i,y_i) \,|\, i \in \{1,\hdots,k\}\}$ and $v$ is the unique value for which $crit(v) \in I$.
\end{proposition}
\begin{proof}
Let $(S_M,S_W)$ is an instance of the SMP with unacceptability and ties. \\
\noindent \fbox{Answer set $\Ra$ Optimal stable set} Let $I$ be an arbitrary answer set of $\PP_{crit}$ and let $S_I$ be as in the proposition. It is clear that the only rules in $\PP_{crit}$ that influence the literals of the form $manpropose(.,.)$, $womanpropose(.,.)$ and $accept(.,.)$ are the rules in $\PP_{norm}$. Hence any answer set $I$ of $\PP_{crit}$ should contain an answer set $I_{norm}$ of $\PP_{norm}$ as a subset. Proposition \ref{pr:SMPASPindifasss} implies that $I_{norm}$ corresponds to a stable set $S_I = \{(m,w) \,|\,$ $accept(m,w)$ $\in$ $I_{norm}\}$. Moreover, the only literals of the form $manpropose(.,.)$, $womanpropose(.,.)$ and $accept(.,.)$ in $I$ are those in $I_{norm}$, so $S_I = \{(m,w) \,|\,$ $accept(m,w)$ $\in$ $I\}$. If $crit=sexeq$, it is straightforward to see that the literals of the form $accept(.,.)$ in $I_{norm}$ uniquely determine which literals of the form $mancost(.,.)$, $womancost(.,.)$, $manweight(.)$, $womanweight(.)$ and $sexeq(.)$ should be in the answer set $I$. These literals do not occur in rules of $\PP_{crit}$ besides those in $\PP^{sexeq}_{ext}$. Notice that the rules which do contain these literals will imply that there will be just one literal of the form $sexeq(.)$ in $I$, namely $sexeq(v)$ with $v$ the sex-equalness cost of $S_I$. Analogous results can be derived for $crit \in \{weight, regret, singles\}$. It remains to be shown that $S_I$ is an optimal stable set. Suppose by contradiction that $S_I$ is not optimal, so there exists a stable set $S^{\ast}$ such that $v_I > v^\ast$, with $v^\ast$ the criterion value of $S^\ast$ to be minimized. We will prove that this implies that $I$ cannot be an answer set of $\PP_{crit}$, contradicting our initial assumption. \\
Proposition \ref{pr:SMPASPindifssas} and Lemma \ref{lem:disj} imply that there exists an interpretation $I^\ast_{disj}$ of the ASP program $\PP_{disj}$ induced by $(S_M,S_W)$ that corresponds to the stable set $S^\ast$. Moreover this interpretation is consistent, i.e.\ it will not contain $atom$ and $\neg atom$ for some atom $atom$. This implies that the interpretation $I'_{disj}$ defined as $I^\ast_{disj}$ in which $\neg atom$ is replaced by $natom$ for every atom $atom$ will falsify the body of the rules of the form (\ref{eq:geenneg}) of $\PP'^{crit}_{ext}$. An analogous reasoning as above yields that the literals of the form $accept'(.,.)$ in $I'_{disj}$ uniquely determine which literals of the form $mancost'(.,.)$, $womancost'(.,.)$, $mansum'(.,.)$, $womansum'(.,.)$, $manweight'(.)$, $womanweight'(.)$ and $sexeq'(.)$ should be in $I'_{disj}$. With those extra literals added to $I'_{disj}$, $I'_{disj}$ satisfies all the rules of $\PP'^{crit}_{ext}$. Moreover, $crit(v^\ast)$ is the unique literal of the form $crit(.)$ in $I'_{disj}$. Notice that $I'_{disj}$ does not contain the atom $sat$.\\
Define the interpretation $J=I_{norm} \cup I'_{disj}$. From the previous argument it follows that $J$ will satisfy every rule of $\PP^{crit}_{ext} \cup \PP'^{crit}_{ext}$ since the predicates occurring in both programs do not overlap. Moreover $J$ contains $crit(v_I)$ and $crit'(v^\ast)$ and these are the only literals of the form $crit(.)$ or $crit'(.)$. Since $v_I > v^\ast$ the rules of the form (\ref{eq:satcrit}) will be satisfied by $J$ since their body is always false. Call $J'$ the set $J \cup \{a \,|\, (a \la) \in \PP_{sat}\}$. Since $J'$ does not contain $sat$, the rules of $\PP_{sat}$ will all be satisfied by $J'$, with exception of the rule $\la not \, sat$. \\
The rule of the form (\ref{eq:sat}) implies that $I$ as answer set of $\PP_{crit}$ should contain $sat$. Now the set of rules (\ref{eq:cost}) -- (\ref{eq:satur}) imply that $I$ should also contain the literals $mancost'(.,.), womancost'(.,.)$ and $manpropose'(.,.)$, $womanpropose'(.,.)$, $accept'(.,.)$ with the corresponding literals prefixed by $n$ for every possible argument stated by the facts in $\PP_{sat}$. The successively computing rules (\ref{eq:weightsucc}) resp. (\ref{eq:regretsucc}) and (\ref{eq:singlessucc}) in $\PP'^{crit}_{ext}$, by which we replaced rules (\ref{eq:weightm}) -- (\ref{eq:weightw}) resp. (\ref{eq:regretm}) -- (\ref{eq:regretw}) and (\ref{eq:card}), garantuee that for every possible set of marriages and its corresponding criterium value $c$ $I$ will contain $crit(c)$ and all associated intermediate results. E.g.\ for $crit=sexeq$, the rules will garantuee that $I$ also contains $mansum'(.,.)$, $manweight(.)$, $womansum(.,.)$ and $womanweight(.)$ for every argument that could occur in a model of $\PP'^{crit}_{ext}$. \footnote[8]{Notice that this would not be the case if we use the original rules with $\#sum$, $\#max$ and $\#count$ in $\PP'^{crit}_{ext}$, since these rules would lead to only one value $c_M$ for which e.g.\ $manweight(c_M)$ should be in $I$, and similarly only one value $c_W$ for which $womanweight(c_W)$ should be in $I$. Consequently there would be only one value $c$ such that $crit(c)$ should be in $I$. This value would not necessarily correspond to $v^\ast$ and so we would not be able to conclude that $I'_{disj} \subseteq I$. Moreover DLV does not allow the use of these rules because of the cyclic dependency of literals they would create, involving the variables in the aggregate functions.}
This implies that $I'_{disj} \subseteq I$. We already reasoned in the beginning of the proof that $I_{norm} \subseteq I$ holds so it follows that $J \subseteq I$. Since the literals of $J' \setminus J$ are stated as facts of $\PP^{crit}_{ext}$, they should be in $I$, hence $J' \subseteq I$. Moreover $J' \subset I$ since $sat \in I \setminus {J'}$. \\
We use the notation $red(\PP,I)$ to denote the reduct of an ASP program $\PP$ w.r.t.\ an interpretation $I$. There is no rule in $\PP'^{crit}_{ext}$ with negation-as-failure in the body, hence $red(\PP'^{crit}_{ext},I) = red(\PP'^{crit}_{ext},J')$ = $\PP'^{crit}_{ext}$. We already reasoned that $J'$ satifies all the rules of the latter. We also reasoned that $I$ does not contain any other literals of the form $accept(.,.)$ than those who are also in $I_{norm}$, and by construction the same holds for $J'$. Hence $red(\PP^{crit}_{ext},I) = red(\PP^{crit}_{ext},J')$ and by construction $J'$ satisfies all the rules of this reduct. It is clear that $red(\PP_{sat},I)$ is $\PP_{sat}$ without the rule $\la not \, sat$, since $sat \in I$. Again we already argued that $J'$ satisfies $red(\PP_{sat},I)$. Hence $J'$ satisfies all the rules of $red(\PP_{crit},I)$, implying that $I$, which strictly contains $J'$, cannot be an answer set of $\PP_{crit}$ since it is not a minimal model of the negation-free ASP program $red(\PP_{crit},I)$~\cite{ASPGel88}. \\
\noindent \fbox{Optimal stable set $\Ra$ Answer set} Let $(S_M,S_W)$ be an instance of the SMP with unacceptability and ties and let $S = \{(x_1,y_1), \hdots, (x_k,y_k)\}$ be an optimal stable set with optimal criterion value $v$. To see that the second part of the proposition holds it suffices to verify that the following interpretation $I$ is an answer set of $\PP_{crit}$, with the notation $P_{x_i}(y)$ as the index $a$ for which $y \in \sigma^l_M(a)$ if $x_i=m_l$ and symmetrically $P_{y_i}(x)$ as the index $a$ for which $x \in \sigma^{l'}_W(a)$ if $y_i=w_{l'}$. If $x_i=y_i$ we set $P_{x_i}(y_i)=P_{y_i}(x_i)=|\sigma^i_M|$ if $x_i$ is a man and $|\sigma^i_W|$ otherwise. So let $I$ be given by:
\allowdisplaybreaks
\begin{align*}
I = I_1 \cup I_2
\end{align*}
with 
\begin{align*}
I_1 = & \{accept (x_i,y_i) \,|\, i \in \{1,\hdots,k\}\} \{crit(v)\} \cup \{sat\} \\
\cup& \{womanpropose(x_i,y_i) \,|\, x_i \neq y_i\} \{manpropose(x_i,y_i)|x_i \neq y_i\} \\
\cup& \{manpropose(x_i,y) \,|\, i \in \{1,\hdots,k\}, x_i=m_l, \exists a < P_{x_i}(y_i) \dpt y \in \sigma^l_M(a)\} \\
\cup& \{womanpropose(x,y_i)|i \in \{1,\hdots,k\}, y_i=w_{l'} ,\exists a < P_{y_i}(x_i) \dpt x \in \sigma^{l'}_W(a)\} \\
\cup& \{mancost(l,P_{x_i}(y_i)) \,|\, crit \neq singles, i \in \{1,\hdots,k\}, x_i = m_l\} \\ 
\cup& \{womancost(P_{y_i}(x_i),l') \,|\, crit \neq singles, i \in \{1,\hdots,k\}, y_i = w_{l'}\} \\
\cup& \{manweight(c_M(S)) \,|\, crit \in \{sexeq,weight\}\} \\  \cup& \{womanweight(c_W(S)) \,|\, crit \in \{sexeq,weight\}\} \\
\cup& \{manregret(c_{regret,M}(S)) \,|\, crit = regret\} \\
\cup& \{womanregret(c_{regret,W}(S)) \,|\, crit = regret\}
\end{align*}
and
\begin{align}
I_2 = &\{manargcost'_1(z) \,|\, z \in \{1,\hdots,n\}\}\nonumber\\  \cup & \{manargcost'_2(z) \,|\, z \in \{1,\hdots,p+1\}\} \nonumber\\
\cup & \{womanargcost'_1(z)  \,|\, z \in \{1,\hdots,p\}\} \nonumber \\ 
\cup & \{womanargcost'_2(z)  \,|\, z \in \{1,\hdots,n+1\}\} \nonumber\\
\cup & \{man(x) \,|\, x \in M\} \cup \{woman(x) \,|\, x \in W\} \label{eq:facts}\\
\cup &\{mancost'(i,j) \,|\, crit \neq singles, i \in \{1,\hdots,n\}, j \in \{1,\hdots,p+1\}\} \nonumber\\
\cup& \{womancost'(j,i) \,|\, crit \neq singles, i \in \{1,\hdots,n+1\}, j \in \{1,\hdots,p\}\} \label{eq:mwcost}\\
\cup& \{manpropose'(x,y) \,|\, x \in M, y \in W\} \cup \{womanpropose'(x,y) \,|\,x \in M, y \in W\}  \nonumber\\
\cup & \{accept'(x,y) \,|\, x \in M, y \in W\} \cup \{accept'(x,x) \,|\,x \in M \cup W\}  \nonumber\\
\cup &\{nmanpropose'(x,y) \,|\, x \in M, y \in W\} \nonumber\\
\cup &\{nwomanpropose'(x,y) \,|\,x \in M, y \in W\} \nonumber\\
\cup &\{naccept'(x,y) \,|\, x \in M, y \in W\} \cup \{naccept'(x,x) \,|\,x \in M \cup W\} \label{eq:mpwpacc}\\
\cup& \{crit'(val) \,|\, val \in \arg(crit)\} \nonumber\\
\cup & \{single'(i,j)\,|\, crit=singles, i\in \{1,\hdots,n+p\}, j \in \{0,1\}\}\nonumber\\
\cup & \{singlesum'(i,j)\,|\, crit=singles, i\in \{1,\hdots,n+p\}, j \in \{1,\hdots,(n+p-i+1)\}\}\nonumber\\
\cup& \{mansum'(i,j) \,|\, crit \in \{sexeq,weight\}, i \in \{1,\hdots,n\}, j \in \{n-i+1,\hdots,(n-i+1)(p+1)\}\} \nonumber\\
\cup& \{womansum'(j,i) \,|\, crit \in \{sexeq,weight\}, j \in \{1,\hdots,p\}, i \in \{p-j+1,\hdots,(p-i+1)(n+1)\}\} \nonumber\\
\cup& \{manweight'(z) \,|\, crit \in \{sexeq,weight\}, z \in \{n,\hdots,n(p+1)\}\} \nonumber\\
\cup& \{womanweight'(z) \,|\, crit \in \{sexeq,weight\}, z \in \{p,\hdots,p(n+1)\}\} \nonumber\\
\cup& \{manmax'(i,j) \,|\, crit = regret, i \in \{1,\hdots,n\}, j \in \{1,\hdots,p+1\}\} \nonumber\\
\cup& \{womanmax'(j,i) \,|\, crit = regret, j \in \{1,\hdots,p\}, i \in \{1,\hdots,n+1\} \nonumber\\
\cup& \{manregret'(z) \,|\, crit=regret, z \in \{1,\hdots,p+1\}\} \nonumber\\
\cup& \{womanregret'(z) \,|\, crit=regret, z \in \{1,\hdots,n+1\}\} \label{eq:rest}
\end{align}  
The notation $\arg(c)$ stands for the possible values the criterion can take within this problem instance:
\begin{itemize}
\item $crit=sexeq \Ra \arg(crit)=\{0,\hdots, \max(np+n-p,np+p-n)\}$,
\item $crit=weight \Ra \arg(crit)=\{n+p,\hdots,2np+p+n\}$,
\item $crit=regret \Ra \arg(crit)=\{1,\hdots,\max(p,n)+1)\}$,
\item $crit=singles \Ra \arg(crit)=\{0,\hdots,n+p\}$.
\end{itemize}
To verify wether this interpretation is an answer set of $\PP_{crit}$, we should compute the reduct w.r.t.\ $I$ and check wether $I$ is a minimal model of the reduct~\cite{ASPGel88}. It can readily be checked that $I$ satisfies all the rules of $red(\PP_{crit},I)$. It remains te be shown that there is no strict subset of $I$ with satisfies all the rules. First of all all the facts of $\PP_{crit}$ must be in the minimal model of the reduct, explaining why the sets of literals (\ref{eq:facts}) should be in $I$. The only rules with negation-as-failure are part of $\PP^{crit}_{ext}$.\\
As in the previous part of the proof, it is straightforward to see that $I_1$ is the unique minimal model of the reduct of $\PP^{crit}_{ext}$ w.r.t.\ $I$, considering that the literals in $I_2$ don't occur in $\PP^{crit}_{ext}$. So any minimal model of $red(\PP_{crit},I)$ must contain $I_1$.\\
The key rule which makes sure that $I$ is a minimal model of the reduct is (\ref{eq:satcrit}). The rules (\ref{eq:geenneg}) imply that for each model of $red(\PP_{crit},I)$ that does not contain $sat$, the literals of $\PP'^{crit}_{ext}$ in that model will correspond to a stable set of the SMP instance. In that case rule (\ref{eq:satcrit}) will have a true body, since $S$ is optimal, implying that $sat$ should have been in the model. And the presence of $sat$ in any minimal model implies the presence of the set of literals (\ref{eq:rest}) in any minimal model of the reduct. This can be seen with the following reasoning. Due to the presence of the facts (\ref{eq:facts}) and $sat$ in any minimal model of the reduct, rules (\ref{eq:cost}) imply the presence of the literals (\ref{eq:mwcost}) in any minimal model. For the same reason rules (\ref{eq:satur}) imply that the literals (\ref{eq:mpwpacc}) should be in any minimal model of $red(\PP_{crit},I)$. For $crit=sexeq$ the presence of the literals of the form (\ref{eq:mwcost}) in any minimal model of the reduct together with rules (\ref{eq:weightsucc}) imply that $mansum'(i,j)$ should be in any minimal model for every $i \in \{1,\hdots,n\}$ and $j \in \{n-i+1,\hdots,(n-i+1)(p+1)\}$: for $i=n$ the first rule of (\ref{eq:weightsucc}) implies that $mansum'(n,x)$ is in any minimal model for every $x$ such that $manargcost_2'(x)$ is in it, i.e.\ any $x \in \{1,\hdots,p+1\}$. Now the second rule of (\ref{eq:weightsucc}) implies that $mansum'(n-1,x)$ is in any minimal model for every $x+y$ such that $manargcost_2'(x)$ and $mansum'(n,y)$ are in it, i.e.\ any $x+y \in \{2,\hdots,2(p+1)\}$. If we continue like this, it is straightforward that every literal of the form $mansum'(.,.)$ of $I_2$ should be in any minimal model. The third rule of (\ref{eq:weightsucc}) now implies that $manweight'(x)$ should be in any minimal model for every $x$ such that $mansum'(1,x)$ is in it, i.e.\ $x \in \{n,\hdots,n(p+1)\}$. The same reasoning can be repeated for the literals $womansum'$ and $womanweight'$. At this point rules (\ref{eq:sexeq}) imply that $sexeq'(|x-y|)$ should be in any minimal model which contains $manweight'(x)$ and $womanweight'(y)$. Notice that only one of the two rules in (\ref{eq:sexeq}) will apply for every $x$ and $y$ since the numerical variables in DLV are positive. Considering the arguments for which $manweight'$ and $womanweight'$ should be in any minimal model, it follows that $sexeq'(x)$ should be in any minimal model for every $x \in \{0,\hdots,\max(p(n+1)-n,n(p+1)-p)\}$, which is exactly $\arg(crit)$. For the other criteria, an analogous reasoning shows that the presence of all literals of $I_2$ is required in any minimal model of the reduct.\\
Considering the fact that we have proved that all literals of $I$ should be in any minimal model of the reduct and $I$ fulfils all the rules of the reduct, we know that $I$ is a minimal model of the reduct and thus an answer set of $\PP_{crit}$.

\end{proof}

If we delete from $\PP_{sexeq}$ the rules (\ref{eq:weightw}) -- (\ref{eq:sexeq}) and replace rule (\ref{eq:satcrit}) by the rule $sat \la manweight(X)$, $manweight'(Y), X \leq Y$, then we obtain the M-optimal stable sets. Analogously we can obtain the W-optimal stable sets.

If a criterion is to be maximized, the symbol $\leq$ in rule (\ref{eq:satcrit}) is simply replaced by $\geq$. E.g.\ for $crit=singles$ we will get minimum cardinality stable sets.

\section{Complexity and Future Work}
The NP-complete decision problem `does there exist a stable set with cardinality $\geq k$ (resp.\ $\leq k$) for an SMP instance with unacceptability and ties with $k$ a positive integer?'~\cite{SMPMan99,SMPMan02} has practical importance, e.g.\ in the National Resident Matching Program~\cite{SMPMan02}. If we add a rule \textit{sat} $\la singles(X), X \leq (n+p-2k)$ to the extended induced program $\PP^{singles}_{ext}$ defined in Subsect.\ \ref{sec:defOSSASP}, then this problem can be formulated as `does there exist an answer set of the normal ASP program $\PP^{singles}_{ext}$ which contains the literal \textit{sat}?' (i.e.\ brave reasoning), another NP-complete problem~\cite{ASPBa03}. So our model forms a suitable framework for these kind of decision problems concerning optimality of stable sets in the SMP.

Notice that the complexity of this kind of decision problem and the one mentioned in the last paragraph of Subsect.\ \ref{sect:SMPinASP} are a good indication how hard it is to find an (optimal) stable set, as opposed to the problems `does there exist an (optimal) stable set?', which tell us how hard it is to know whether there exists a solution but not necessarily how hard it is to find one.

Combining these problems leads to a new decision problem: `is the pair $(m,w)$ \textit{optimally stable} for an instance of the SMP with unacceptability and ties?'. We define an \textit{optimally stable pair} as a pair $(m,w)$ for which there exists an optimal stable set in which $m$ and $w$ are matched. As far as we know this problem has not been studied yet, although it could be useful in practice, for instance if one wants to find a maximum cardinality matching but also wants to prioritize some couple or a person. Optimality is still desirable, because it ensures the others from not being put too much at a disadvantage. For instance in the kidney exchange problem, in which kidney patients with a willing but incompatible donor try to interchange each other's donors to get a transplant, this is a realistic situation: if two patients with intercompatible donors urgently need a transplant, they should get priority, but of course we still want to match as many patients to donors as possible. Considering the complexity of the separate decision problems, the combined problem might have a higher complexity, perhaps corresponding to the $\Sigma^P_2$-complexity of our grounded disjunctive normal ASP program with aggregate functions~\cite{ASPBa03,ASPDell03}. It should be noticed however that the addition of constraints not necessarily increases complexity and a precise classification of complexity is desirable.

\section{Conclusion}
We formalized and solved different variants of the SMP using ASP programs, which can easily be adapted to yet other variants. Moreover we applied saturation to compute optimal stable sets, with the advantage that these programs can be handled with the efficient off-the-shelf ASP solver DLV. To the best of our knowledge, our encoding offers the first exact implementation of finding sex-equal, egalitarian, minimum regret, or maximum cardinality stable sets for an instance of the SMP with unacceptability and ties. Hence, our general framework allows us to tackle a class of problems and requires only small adaptions to easily shift between them.

\bibliographystyle{plain}
\bibliography{biblio}

\end{document}